\documentclass[11pt,a4paper]{article}
\usepackage[hyperref]{emnlp2018}
\usepackage{times}
\usepackage{latexsym}
\usepackage{amsmath}
\usepackage{amssymb}
\usepackage{amsthm}
\usepackage{url}
\usepackage{graphicx}
\usepackage[utf8]{inputenc}
\usepackage[english]{babel}
\usepackage{booktabs}
\usepackage{multirow}
\usepackage{color}
\usepackage{algorithm}
\usepackage[noend]{algpseudocode}
\usepackage{dsfont}
\usepackage{url}

\newtheorem{theorem}{Theorem}

\aclfinalcopy 

\title{Unsupervised Cross-lingual Transfer of Word Embedding Spaces}

\author{Ruochen Xu, Yiming Yang, Naoki Otani, Yuexin Wu \\
  Carnegie Mellon University \\
  {\tt \{ruochenx, yiming, notani, yuexinw \}@cs.cmu.edu} \\
  }

\date{}

\begin{document}
\maketitle
\begin{abstract}






Cross-lingual transfer of word embeddings aims to establish the semantic mappings among words in different languages by learning the transformation functions over the corresponding word embedding spaces. Successfully solving this problem would benefit many downstream tasks such as to translate text classification models from resource-rich languages (e.g. English) to low-resource languages. Supervised methods for this problem rely on the availability of cross-lingual supervision, either using parallel corpora or bilingual lexicons as the labeled data for training, which may not be available for many low resource languages. This paper proposes an unsupervised learning approach that does not require any cross-lingual labeled data. Given two monolingual word embedding spaces for any language pair, our algorithm optimizes the transformation functions in both directions simultaneously based on distributional matching as well as minimizing the back-translation losses. We use a neural network implementation to calculate the Sinkhorn distance, a well-defined distributional similarity measure, and optimize our objective through back-propagation. Our evaluation on benchmark datasets for bilingual lexicon induction and cross-lingual word similarity prediction shows stronger or competitive performance of the proposed method compared to other state-of-the-art supervised and unsupervised baseline methods over many language pairs. 
\end{abstract}

\section{Introduction}

Word embeddings are well known to capture meaningful representations of words based on large text corpora ~\cite{mikolov2013efficient, pennington2014glove}. Training word vectors using monolingual corpora is a common practice in various NLP tasks. 
However, how to establish cross-lingual semantic mapping among monolingual embeddings remain an open challenge as the availability of resources and benchmarks are highly imbalanced across languages.

Recently, increasing effort of research has been motivated to address this challenge. 
Successful cross-lingual word mapping will benefit many cross-lingual learning tasks, such as transforming text classification
models trained in resource-rich languages to low-resource languages. Downstream applications include word alignment, text classification, named entity recognition, dependency parsing, POS-tagging, and more \cite{sogaard2015inverted}. 
Most methods for cross-lingual transfer of word embeddings are based on supervised or semi-supervised learning, i.e., they require cross-lingual supervision such as human-annotated bilingual lexicons and parallel corpora~\cite{lu2015deep, smith2017offline, artetxe2016learning}. 
Such a requirement may not be met for many language pairs in the real world. 

This paper proposes an unsupervised approach to the cross-lingual transfer of monolingual word embeddings, which requires zero cross-lingual supervision. The key idea is to optimize the mapping in both directions for each language pair (say A and B), in the way that the word embedding translated from language A to language B will match the distribution of word embedding in language B. And when translated back from B to A, the word embedding after two steps of transfer will be maximally close to the original word embedding. A similar property holds for the other direction of the loop (from B to A and then from A back to B).
Specifically, we use the Sinkhorn distance~\cite{cuturi2013sinkhorn} to capture the distributional similarity between two set of embeddings after transformation, which we found empirically superior to the KL-divergence~\cite{zhang2017adversarial} and distance to nearest neighbor \cite{artetxe2017learning,conneau2017word} with regards to the quality of learned transformation as well as the robustness under different training conditions.

Our novel contributions in the proposed work include:
\begin{itemize}
  \setlength{\parskip}{0in}
  \setlength{\itemsep}{0in}
\item We propose an unsupervised learning framework which incorporates the Sinkhorn distance as a distributional similarity measure in the back-translation loss function.
\item We use a neural network to optimize our model, especially to implement the Sinkhorn distance whose calculation itself is an optimization problem. 
\item Unlike previous models which only consider cross-lingual transformation in a single direction, our model jointly learns the word embedding transfer in both directions for each language pair. 
\item We present an intensive comparative evaluation where our model achieved the state-of-the-art performance for many language pairs in cross-lingual tasks.  
\end{itemize}




\section{Related Work}
We divide the related work into supervised and unsupervised categories.  Representative methods in both categories are included in our comparative evaluation (Section \ref{sec:impl}). We also discuss some related work in unsupervised domain transfer in addition. 
\vskip 0.1in
\noindent\textbf{Supervised Methods}:
There is a rich body of supervised methods for learning cross-lingual transfer of word embeddings based on bilingual dictionaries \cite{mikolov2013efficient, faruqui2014improving, artetxe2016learning, xing2015normalized, duong2016learning, gouws2015simple}, sentence-aligned corpora \cite{kovcisky2014learning, hermann2014multilingual, gouws2015bilbowa} and document-aligned corpora \cite{vulic2016bilingual, sogaard2015inverted}. The most relevant line of work is that by \citet{mikolov2013efficient} where they showed monolingual word embeddings are likely to share similar geometric properties across languages although they are trained separately and hence cross-lingual mapping can be captured by a linear transformation across embedding spaces. Several follow-up studies tried to improve the cross-lingual transformation in various ways \cite{faruqui2014improving, artetxe2016learning, xing2015normalized, duong2016learning, ammar2016massively, artetxe2016learning, zhang2016ten, shigeto2015ridge}. Nevertheless, all these methods require bilingual lexicons for supervised learning. \citet{vulic2016role} showed that
5000 high-quality bilingual lexicons are sufficient for learning a reasonable cross-lingual mapping. 
\vskip 0.1in
\noindent\textbf{Unsupervised Methods}
have been studied to establish cross-lingual mapping without any human-annotated supervision. 
Earlier work simply relied on word occurrence information only \cite{rapp1995identifying,fung1995compiling} while later efforts have considered more sophisticated statistics in addition \cite{haghighi2008learning}. 
The main difficulty in unsupervised learning of cross-lingual mapping is the formulation of the objective function, i.e., how to measure the goodness of an induced mapping without any supervision is a non-trivial question. 
\citet{cao2016distribution} tried to match the mean and standard deviation of the embedded word vectors in two different languages after mapping the words in the source language to the target language. However, such an approach has shown to be sub-optimal because the objective function only carries the first and second order statistics of the mapping. \citet{artetxe2017learning} tried to impose an orthogonal constraint to their linear transformation model and minimize the distance between the transferred source-word embedding and its nearest neighbor in the target embedding space. Their method, however, requires a seed bilingual dictionary as the labeled training data and hence is not fully unsupervised. \cite{zhang2017adversarial,barone2016towards} adapted a generative adversarial network (GAN) to make the transferred embedding of each source-language word indistinguishable from its true translation in the target embedding space \cite{goodfellow2014generative}. The adversarial model could be optimized in a purely unsupervised manner but is often suffered from unstable training, i.e. the adversarial learning does not always improve the performance over simpler baselines.
\citet{zhang2017earth}, \citet{conneau2017word} and \citet{artetxe2017learning} also tried adversarial approaches for the induction of seed bilingual dictionaries, as a sub-problem in the cross-lingual transfer of word embedding. 

\vskip 0.1in
\noindent\textbf{Unsupervised Domain Transfer}:
Generally speaking, learning the cross-lingual transfer of word embedding can be viewed as a domain transfer problem, where the domains are word sets in different languages. Thus various work in the field of \textit{unsupervised domain adaptation} or \textit{unsupervised transfer learning} can shed light on our problem.
For example, \citet{he2016dual} proposed a semi-supervised method for machine translation to utilize large monolingual corpora. \citet{shen2017style} used unsupervised learning to transfer sentences of different sentiments. Recent work in computer vision addresses the problem of image style transfer without any annotated training data~\cite{zhu2017unpaired, taigman2016unsupervised, yi2017dualgan}. Among those, our work is mostly inspired by the work on CycleGAN \cite{zhu2017unpaired}, and we adopt their cycled consistent loss over images into our back-translation loss. One key difference of our method from CycleGAN is that they used the training loss of an adversarial classifier as an indicator of the distributional distance, but instead, we introduce the Sinkhorn distance in our objective function and demonstrate its superiority over the representative method using adversarial loss \cite{zhang2017adversarial}.

\section{Proposed Method}
Our system takes two sets of monolingual word embeddings of dimension $d$ as input, which are trained separately on two languages. We denote them as $X=\{x_i\}_{i=1}^{n}$,  $Y=\{y_j\}_{j=1}^{m}$, $ x_i, y_j \in \mathbb{R}^d $. 
During the training of monolingual word embedding for $X$ and $Y$, we also have the access to the word frequencies, represented by vectors $r \in \mathbb{N}^n$ and $c \in \mathbb{N}^m$ for $X$ and $Y$, respectively. Specifically, $r_i$ is the frequency for word (embedding) $x_i$ and similarly for $c_j$ of $y_j$. As illustrated in Figure \ref{fig:sys_arch}, our model has two mappings: $G: X \rightarrow Y$ and $F: Y \rightarrow X$. We further denote transferred embedding from $X$ as $G(X) := \{G(x_i)\}_{i=1}^{n}$ and correspondingly for $F(Y)$. 

\begin{figure}[ht]
\label{fig:sys_arch}
\centering
\includegraphics[width=0.5\textwidth]{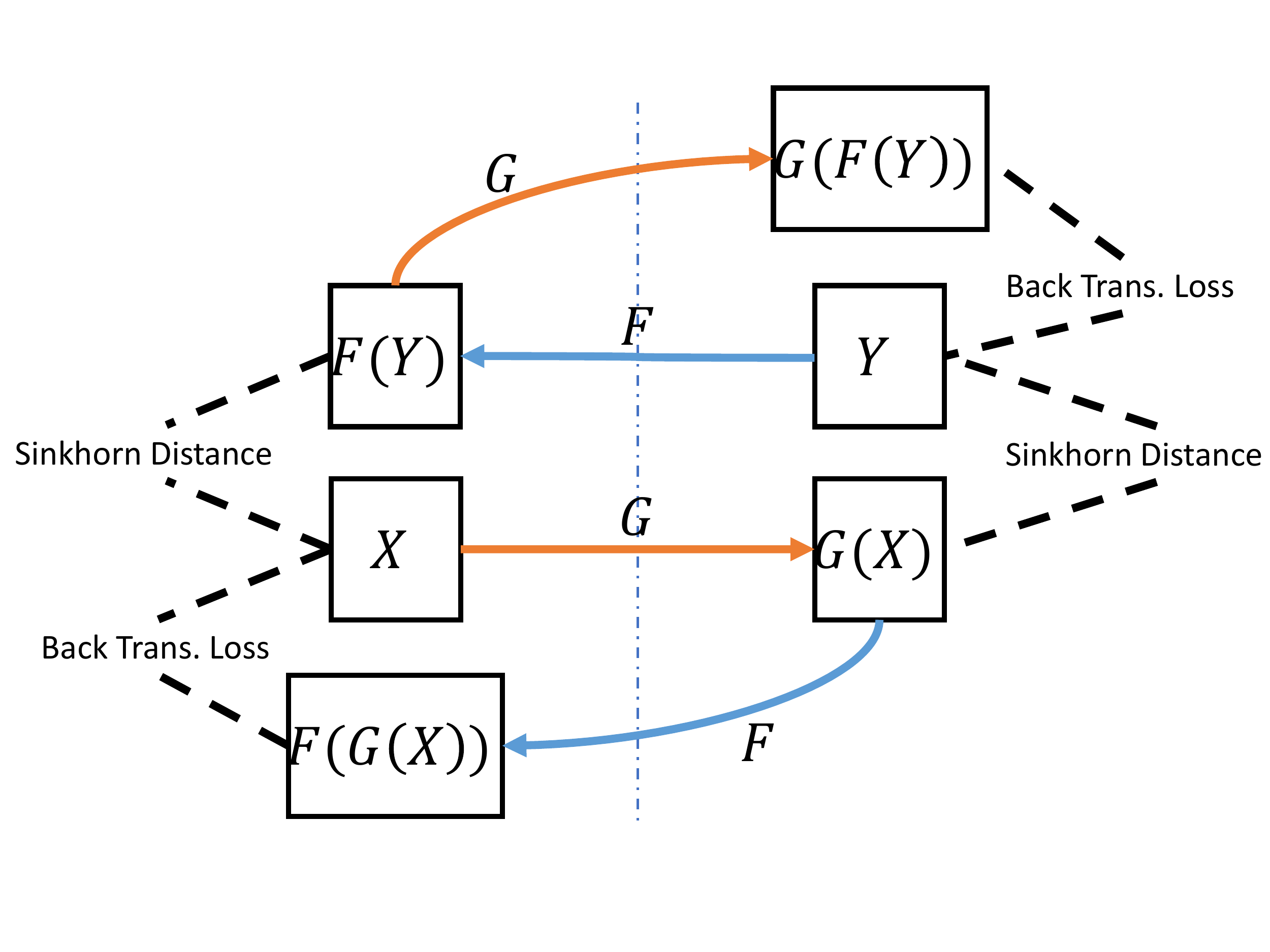}
\caption{The model takes monolingual word embedding $X$ and $Y$ as input. $G$ and $F$ are embedding transfer functions parameterized by a neural network, which are represented by solid arrows. The dashed lines indicate the input for our objective losses, namely the Sinkhorn distance and back-translation loss}.
\end{figure}

In the unsupervised setting, the goal is to learn the mapping $G$ and $F$ without any paired word translation. To achieve this, our loss function consists of two parts: Sinkhorn distance \cite{cuturi2013sinkhorn} for matching the distribution of transferred embedding to its target embedding distribution; and a back-translation loss for preventing degenerated transformation.

\subsection{Sinkhorn Distance}
\subsubsection{Definition}
\label{sec:sk-dist-def}
Sinkhorn distance is a recently proposed distance between probability distributions. We use the Sinkhorn distance to measure the closeness between $G(X)$ and $Y$, and also between $F(Y)$ and $X$. During the training, our model optimizes $G$ and $F$ for lower Sinkhorn distance to make the transferred embeddings match the distribution of the target embeddings. Here we only illustrate the Sinkhorn distance between $G(X)$ and $Y$, the derivation for $F(Y)$ and $X$ is very similar. Although the vocabulary sizes of two languages could be different, we are able to sample mini-batches of equal size from $G(X)$ and $Y$. therefore we assume $n=m$ in the following derivation.

To compute Sinkhorn distance, we firstly compute a distance matrix $M^{(G)} \in \mathbb{R}^{n \times m}$ between $G(X)$ and $Y$ where $M^{(G)}_{ij}$ is the distance measure between $G(x_i)$ and $y_j$. The superscript on $M^{(G)}$ indicates the distance that depends on a parameterized transformation $G$. For instance, if we choose Euclidean distance as a measure (see Section  \ref{sec:dist choice} for more discussions), we will have $$M^{(G)}_{ij} = \| G(x_i) - y_j \|_{2.} $$


Given the distance matrix, the Sinkhorn distance between $P_{G(X)}$ and $P_{Y}$  is defined as:
\begin{equation}
\label{eq:sk-dist-def}
d_{sh}(G) := \min\limits_{P \in U_\alpha(r,c)} \langle P, M^{(G)} \rangle
\end{equation} where $\langle \cdot , \cdot \rangle$ is the Forbenius dot-product and $U_\alpha(r,c)$ is an entropy constrained transport polytope, defined as
\begin{align}
\label{eq:entr-trans-poly}
U_\alpha(r,c) = \{ & P \in \mathbb{R_{+}}^{n \times m} | P \mathds{1}_{m} = r,  P^T \mathds{1}_{n} = c,\nonumber \\
& h(P) \leq h(r) + h(c) - \alpha \} 
\end{align} 
Note that $P$ is non-negative and the first two constraints make its element-wise sum be~$1$. Therefore, $P$ can be seen as a set of probability distributions. The same applies for $r$ and $c$ since they are frequencies.
$h$ is the entropy function defined on any probability distributions and $\alpha$ is a hyperparameter to choose. 
For any probabilistic matrix $P \in U_\alpha(r,c)$, it can be viewed as the joint probability of $(G(X), Y)$. The first two constraints ensure that $P$ has marginal distribution on $G(X)$ as $P_{G(X)}$ and on $Y$ as $P_{Y}$. We can also view $P_{ij}$ as the evidence for establishing a translation between word vector $x_i$ and word vector $y_j$. 

An intuitive interpretation of equation \eqref{eq:sk-dist-def} is that we are trying to find the optimal transport probability $P$ under the entropy constraint such that the total distance to transport from $G(X)$ to $Y$ is minimized.

\begin{algorithm}[t]
\caption{Computation of Sinkhorn Distance $d_{sh}(G)$}\label{alg:sinkhorn-dist}
\begin{algorithmic}[1]
\Procedure{Sinkhorn}{$M^{(G)}, r, c, \lambda, I$}
\State $K^{(G)} := e^{-\lambda M^{(G)}}$
\State $v = \mathds{1}_{m}/m$ \Comment{normalized one vector}
\State $i = 0$
\While{$i<I$}\Comment{iterate for $I$ times}
\State $u = r./K^{(G)}v$
\State $v = c./{K^{(G)}}^Tu$
\State $i = i + 1$
\EndWhile\label{euclidendwhile}
\State $d_{sh}(G)=u^T (( K^{(G)} \otimes M^{(G)} ) v )$
\State \textbf{return} $d_{sh}(G)$\Comment{The Sinkhorn distance}
\EndProcedure
\end{algorithmic}
\end{algorithm}

\subsubsection{Computing Sinkhorn Distance $d_{sh}(G)$}
\label{sec:matrix_scal}
\citet{cuturi2013sinkhorn} showed that the optimal solution of formula \eqref{eq:sk-dist-def} has the form $P^* = \mathbf{diag}(u) K \mathbf{diag}(v)$ , where $u$ and $v$ are some non-negative vectors and $K^{(G)} := e^{-\lambda M^{(G)}}$; $\lambda$ is the Lagrange multiplier for the entropic constraint in \ref{eq:entr-trans-poly} and each $\alpha$ in Equation~\eqref{eq:sk-dist-def} has one corresponding $\lambda$. The Sinkhorn distance can be efficiently computed by a matrix scaling algorithm. We present the pseudo code in Algorithm~\ref{alg:sinkhorn-dist}.
Note that the computation of $d_{sh}(G)$ only requires matrix-vector multiplication. Therefore, we can compute and back propagate the gradient of $d_{sh}(G)$ with regards to the parameters in $G$ using standard deep learning libraries. We show our implementation details in Section~\ref{sec:impl} and supplementary material.

\subsubsection{Choice of the Distance Metric}\label{sec:dist choice}
In Section~\ref{sec:sk-dist-def}, we used the Euclidean distance of vector pairs to define $M^{(G)}$ and Sinkhorn distance $d_{sh}(G)$. However, in our preliminary experiment, we found that Euclidean distance of unnormalized vectors gave poor performance. Therefore, following the common practice, we normalize all word embedding vectors to have a unit L2 norm in the construction of $M^{(G)}$.

As pointed out in Theorem 1 of~\citet{cuturi2013sinkhorn}, $M^{(G)}$ must be a valid metric in order to make $d_{sh}(G)$ a valid metric. For example, the commonly used cosine distance, which is defined as $CosDist(a, b) = 1 - cos(a, b)$, is not a valid metric because it does not satisfy triangle inequality \footnote{If we select $a=[1,0], b=[\frac{\sqrt{2}}{2}, \frac{\sqrt{2}}{2}], c=[0,1] $ We have $ CosDist(a, c) \geq CosDist(a, b) + CosDist(b, c) $ , which violates the triangle inequality.}. Thus, for constructing $M^{(G)}$, we propose the square root cosine distance ($SqrtCosDist$) below:
\begin{align}
SqrtCosDist(a, b) := \sqrt{2-2cos(a, b)} \\
M^{(G)}_{ij} = SqrtCosDist(G(x_i), y_j)  
\end{align}
\begin{theorem}
$SqrtCosDist$ is a valid metric.
\end{theorem}
\begin{proof}
$\forall a, b \in \mathbb{R}^d$, let $\hat{a} = \frac{a}{\|a\|}$, $\hat{b} = \frac{b}{\|b\|}$. We have $cos(a,b) = \langle \hat{a}, \hat{b} \rangle$ and $\langle \hat{a}, \hat{a} \rangle = \langle \hat{b}, \hat{b} \rangle = 1$. Then
\begin{align*}
& SqrtCosDist(a, b)  =  \sqrt{2-2cos(a, b)} \\
& = \sqrt{\langle \hat{a}, \hat{a} \rangle + \langle \hat{b}, \hat{b} \rangle - 2\langle \hat{a}, \hat{b} \rangle } \\
& = \sqrt{\langle \hat{a} - \hat{b}, \hat{a} - \hat{b} \rangle} \\
& = \| \hat{a} - \hat{b} \|
\end{align*}
Obviously, the last term is the Euclidean distance between normalized input vectors $\hat{a}$ and $\hat{b}$. Since Euclidean distance is a valid metric, it follows that $SqrtCosDist$ satisfies all the axioms for a valid metric. 
\end{proof}

\subsection{Objective Function}
Given enough capacity, $G$ is capable to transfer $X$ to $Y$ for arbitrary word-to-word mappings. To ensure that, we learn a meaningful translation and also to regularize the search space of possible transformations, we enforce the word embedding after the forward and the backward transformation should not diverge much from its original direction. We simply choose the back-translation loss based on the cosine similarity:
\begin{align}
\label{eq:back-trans}
d_{bt}(G,F) = & \sum_{i} 1 - cos(x_i, F(G(x_i))) + \nonumber \\  
 & \sum_{j} 1 - cos(y_i, G(F(y_i)))
\end{align} where $cos$ is the cosine similarity.

Putting everything together, we minimize the following objective function. 
\begin{equation}
\label{eq:overall-obj}
L_{X,Y,r,c}(G, F) = d_{sh}(G) + d_{sh}(F) + \beta d_{bt}(G,F)
\end{equation} where hyper-parameter $\beta$ controls the relative weight of the last term against the first two terms in the objective function. By definition, computation of $d_{sh}(G)$ or $d_{sh}(F)$ involves another minimization problem as shown in Equation~\eqref{eq:sk-dist-def}. We solve it using the matrix scaling algorithm in Section~\ref{sec:matrix_scal}, and treat $d_{sh}(G)$ as a deterministic and differentiable function of parameters in $G$. The same holds for $d_{sh}(F)$ and $F$.

\subsection{Wasserstein GAN Training for Good Initial Point}
\label{subsec:WGAN}
In preliminary experiments, we found that our objective \ref{eq:overall-obj} is sensitive to the initialization of the weight in $G$ and $F$ in the purely unsupervised setting. It requires a good initial setting of the parameters to avoid getting stuck in the poor local minimal. To address this sensitivity issue, we employed a similar approach as in \cite{zhang2017earth,aldarmaki2018unsupervised} to firstly used an adversarial training approach to learn $G$ and $F$ and use them as the initial point for training our full objective \ref{eq:overall-obj}. More specifically, we choose to minimize the optimal transport distance below.
\begin{equation}
\label{eq:ot-dist-def}
d_{ot}(G) := \min\limits_{P \in U(r,c)} \langle P, M^{(G)} \rangle
\end{equation} $U$ is the transport polytope without entropy constraint, defined as follows.
\begin{align}
\label{eq:trans-poly}
U = \{ & P \in \mathbb{R_{+}}^{n \times m} | P \mathds{1}_{m} = r,  P^T \mathds{1}_{n} = c \} 
\end{align} 

We optimize the distance above by its dual form and through adversarial training, which is also known as Wasserstein GAN~(WGAN)~\cite{arjovsky2017wasserstein}. We applied the optimization trick proposed by~\citet{gulrajani2017improved}.

Although the first phase of adversarial training could be unstable, and the performance is lower than using the Sinkhorn distance, the adversarial training narrows down the search space of model parameters and boosting the training of our proposed model.

\subsection{Implementation} \label{sec:impl}
We implemented transformation $G$ and $F$ by a linear transformation. The dimension of the input and output are the same with the word embedding dimension $d$.\footnote{We tried more complex non-linear transformations for $G$ and $F$. The performance is slightly worse than the linear case.} For all the experiments in the subsequent section, the $\beta$ in \eqref{eq:overall-obj} was set to be $0.1$. For hyper-parameters from the computation of Sinkhorn distance, we choose $\lambda = 10$ and run the matrix scaling algorithm for $20$ iterations. Due to the space constraint, a detailed implementation description is presented in the supplementary material. The code of our implementation is publicly available \footnote{Our implementation \url{https://github.com/xrc10/unsup-cross-lingual-embedding-transfer}}.



\section{Experiments}


We conducted an evaluation of our approach in comparison with state-of-the-art supervised/unsupervised methods on several evaluation benchmarks for bilingual lexicon induction (Task 1) and word similarity prediction (Task 2). We include our main results in this section and report the ablation study in the supplementary material.

\subsection{Data}\label{data}
\subsubsection{Monolingual Word Embedding Data}
All the methods being evaluated in both tasks take monolingual word embedding in each language as the input data. We use publicly available pre-trained word embeddings trained on Wikipedia articles: (1) a smaller set of word embeddings of dimension $50$ trained on comparable Wikipedia dump in five languages~\cite{zhang2017adversarial}\footnote{Available at  \url{http://nlp.csai.tsinghua.edu.cn/~zm/UBiLexAT}} and (2) a larger set of word embeddings of dimension $300$ trained on Wikipedia dump in 294 languages~\cite{bojanowski2016enriching}\footnote{Available at \url{https://github.com/facebookresearch/fastText/blob/master/pretrained-vectors.md}}. For convenience, we name the two sets \textbf{WE-Z} and \textbf{WE-C}, respectively. 
%
\subsubsection{Bilingual Lexicon Data}
We need true translation pairs of words for evaluating methods in bilingual lexicon induction (Task~1). We followed previous studies and prepared two datasets below.

\vskip 0.1in
\noindent\textbf{LEX-Z}: \citet{zhang2017adversarial} constructed the bilingual lexicons from various resources. Since their ground truth word pairs are not released, we followed their procedure, crawled bilingual dictionaries and randomly separated them into the training and testing set of equal size.\footnote{The bilingual dictionaries we crawled are submitted as supplementary material.} Note that our proposed method did not utilize the training set. It was only used by supervised baseline methods described in Section~\ref{subsec:baseline}. There are eight language pairs (order counted); the corresponding dataset statistics are summarized in Table~\ref{tab:UnBiLex-stats}. We use WE-Z embeddings in this dataset.

\vskip 0.1in
\noindent\textbf{LEX-C}: This lexicon was constructed by~\citet{conneau2017word} and contains more translation pairs than LEX-Z. They divided them into training and testing set. We run our model and the baseline methods on 16 language pairs. For each language pair, the training set contains $5,000$ unique query words and the testing set has $1,500$ query words. We followed~\citet{conneau2017word} and set the search space of candidate translations to be the $200,000$ most frequent words in each target language. We use WE-C embeddings in this dataset.

\begin{table}[t]
\centering
\begin{tabular}{@{}llrrr@{}}
\toprule
 &  & \# tokens & vocab. size & bi. lex. size \\ \midrule
\multirow{2}{*}{tr-en} & tr & 6m & 7,482 & 18,404 \\
 & en & 28m & 13,220 & 27,327 \\ \midrule
\multirow{2}{*}{es-en} & es & 61m & 4,774 & 3,482 \\
 & en & 95m & 6,637 & 10,772 \\ \midrule
\multirow{2}{*}{zh-en} & zh & 21m & 3,349 & 54,170 \\
 & en & 53m & 5,154 & 51,375 \\ \midrule
\multirow{2}{*}{it-en} & it & 73m & 8,490 & 4,999 \\
 & en & 93m & 6,597 & 11,812 \\ \bottomrule
\end{tabular} 
\caption{The statistics of LEX-Z. The languages are Spanish~(es), French~(fr), Chinese~(zh), Turkish~(tr) and English~(en). Number of tokens is the size of training corpus of WE-Z. The bilingual lexicon size means the number of unique words of a language in the gold bilingual lexicons.
}
\label{tab:UnBiLex-stats}
\end{table}

\subsubsection{Bilingual Word Similarity Data}

For bilingual word similarity prediction (Task 2) we need the true labels for evaluation. Following \citet{conneau2017word}, we used the SemEval~2017 competition dataset, where human annotators measured the cross-lingual similarity of nominal word pairs according to the five-point Likert scale. This dataset contains word pairs across five languages: English~(en), German~(de), Spanish~(es), Italian~(it), and Farsi~(fa). Each language pair has about 1,000 word pairs annotated with a real similarity score ranging from $0$ to $4$.

\begin{table*}[ht]
\centering
\begin{tabular}{@{}llrrrrrrrr@{}}
\toprule
 & \begin{tabular}[c]{@{}l@{}}Methods\end{tabular} & tr-en & en-tr & es-en & en-es & zh-en & en-zh & it-en & en-it \\ \midrule
\multirow{6}{*}{Supervised} & \citet{mikolov2013efficient} & 19.41 & 10.81 & 68.73 & 41.19 & 45.88 & \textbf{45.37} & 59.83 & \textbf{41.26} \\
 & \citet{zhang2016ten} & 23.39 & \textbf{11.07} & 72.36 & 41.19 & 48.01 & 42.66 & \textbf{63.19} & 40.37 \\
 & \citet{xing2015normalized} & 24.00 & 10.78 & 71.92 & 41.02 & 48.10 & 42.90 & 62.81 & 40.43 \\
 & \citet{shigeto2015ridge} & \textbf{26.56} & 8.52 & 72.23 & 37.80 & \textbf{49.95} & 38.15 & 63.14 & 35.63 \\
 & \citet{artetxe2016learning} & 23.49 & 10.74 & 71.98 & 41.12 & 48.01 & 42.66 & 63.14 & 40.28 \\
 & \citet{artetxe2017learning} & 22.88 & 10.78 & \textbf{72.61} & \textbf{41.62} & 47.54 & 42.82 & 61.32 & 39.63 \\ \midrule
\multirow{3}{*}{Unsupervised} & \citet{conneau2017word} & 4.09 & 1.41 & 60.16 & 33.58 & 41.98 & 34.70 & 26.98 & 15.47 \\
 & \citet{zhang2017adversarial} & 15.83 & 7.41 & 63.41 & 37.73 & 42.08 & 41.26 & 54.75 & 37.17 \\
 & Ours & \textbf{23.29} & \textbf{9.96} & \textbf{73.05} & \textbf{41.95} & \textbf{49.03} & \textbf{44.63} & \textbf{61.42} & \textbf{39.63} \\ \bottomrule
\end{tabular}
\caption{The accuracy@k scores of all methods in bilingual lexicon induction on \textbf{LEX-Z}. The best score for each language pair is bold-faced for the supervised and unsupervised categories, respectively. Language pair "A-B" means query words are in language A and the search space of word translations is in language B. Languages are paired among \textbf{English(en), Turkish~(tr), Spanish~(es), Chinese~(zh) and Italian~(it)}.}
\label{tab:UnBiLex}
\end{table*}

\subsection{Baseline Methods}\label{data}
\label{subsec:baseline}
We evaluated the same set of supervised and unsupervised baselines for comparative evaluation in both Task 1 and Task 2. The supervised baselines  include the methods of \citet{shigeto2015ridge, zhang2016ten, artetxe2016learning, xing2015normalized,mikolov2013efficient,artetxe2017learning}.\footnote{The implementations are available from \url{https://github.com/artetxem/vecmap.}} We fed all the supervised methods with the bilingual dictionaries in the training portions of the LEX-Z and LEX-C datasets, respectively.

For unsupervised baselines we include the methods of \citet{zhang2017adversarial} and \citet{conneau2017word}, whose source code is publicly available as provided by the authors.\footnote{We used implementation by \citet{zhang2017adversarial} from \url{http://nlp.csai.tsinghua.edu.cn/~zm/UBiLexAT} and that of \citet{conneau2017word} from \url{https://github.com/facebookresearch/MUSE}}

\begin{table*}[h]
\centering
\begin{tabular}{@{}llllllllll@{}}
\toprule
 & \begin{tabular}[c]{@{}l@{}}Methods\end{tabular} & bg-en & en-bg & ca-en & en-ca & sv-en & en-sv & lv-en & en-lv \\ \midrule
\multirow{6}{*}{Supervised} & \citet{mikolov2013efficient} & 44.80 & \textbf{48.47} & 57.73 & \textbf{66.20} & 43.73 & \textbf{63.73} & 26.53 & \textbf{28.93} \\
 & \citet{zhang2016ten} & 50.60 & 39.73 & 63.40 & 58.73 & 50.87 & 53.93 & 34.53 & 22.87 \\
 & \citet{xing2015normalized} & 50.33 & 40.00 & 63.40 & 58.53 & 51.13 & 53.73 & 34.27 & 21.60 \\
 & \citet{shigeto2015ridge} & \textbf{61.00} & 33.80 & \textbf{69.33} & 53.60 & \textbf{61.27} & 41.67 & \textbf{42.20} & 13.87 \\
 & \citet{artetxe2016learning} & 53.27 & 43.40 & 65.27 & 60.87 & 54.07 & 55.93 & 35.80 & 26.47 \\
 & \citet{artetxe2017learning} & 47.27 & 34.40 & 61.27 & 56.73 & 38.07 & 44.20 & 24.07 & 12.20 \\ \midrule
\multirow{3}{*}{Unsupervised} & \citet{conneau2017word} & 26.47 & 13.87 & 41.00 & 33.07 & 24.27 & 24.47 & - & - \\
 & \citet{zhang2017adversarial} & - & - & - & - & - & - & - & - \\
 & Ours & \textbf{50.33} & \textbf{34.27} & \textbf{58.60} & \textbf{54.60} & \textbf{48.13} & \textbf{50.47} & \textbf{27.73} & \textbf{13.53} \\ \bottomrule
\end{tabular}
\caption{The accuracy@k scores of all methods in bilingual lexicon induction on \textbf{LEX-C}. The best score for each language pair is bold-faced for the supervised and unsupervised categories, respectively. Languages are paired among \textbf{English(en), Bulgarian(bg), Catalan(ca), Swedish(sv) and Latvian(lv)}. "-" means that during the training time, the model failed to converge to reasonable local minimal and hence the result is omitted in the table.}
\label{tab:MUSE-low}
\end{table*}

\begin{table*}[ht]
\centering
\begin{tabular}{@{}llllllllll@{}}
\toprule
 & \begin{tabular}[c]{@{}l@{}}Methods\end{tabular} & de-en & en-de & es-en & en-es & fr-en & en-fr & it-en & en-it \\ \midrule
\multirow{6}{*}{Supervised} & \citet{mikolov2013efficient} & 61.93 & \textbf{73.07} & 74.00 & \textbf{80.73} & 71.33 & \textbf{82.20} & 68.93 & \textbf{77.60} \\
 & \citet{zhang2016ten} & 67.67 & 69.87 & 77.27 & 78.53 & 76.07 & 78.20 & 72.40 & 73.40 \\
 & \citet{xing2015normalized} & 67.73 & 69.53 & 77.20 & 78.60 & 76.33 & 78.67 & 72.00 & 73.33 \\
 & \citet{shigeto2015ridge} & \textbf{71.07} & 63.73 & \textbf{81.07} & 74.53 & \textbf{79.93} & 73.13 & \textbf{76.47} & 68.13 \\
 & \citet{artetxe2016learning} & 69.13 & 72.13 & 78.27 & 80.07 & 77.73 & 79.20 & 73.60 & 74.47 \\
 & \citet{artetxe2017learning} & 68.07 & 69.20 & 75.60 & 78.20 & 74.47 & 77.67 & 70.53 & 71.67 \\ \midrule
Unsupervised & \citet{conneau2017word} & \textbf{69.87} & \textbf{71.53} & \textbf{78.53} & 79.40 & \textbf{77.67} & \textbf{78.33} & \textbf{74.60} & \textbf{75.80} \\
 & \citet{zhang2017adversarial} & - & - & - & - & - & - & - & - \\
 & Ours & 67.00 & 69.33 & 77.80 & \textbf{79.53} & 75.47 & 77.93 & 72.60 & 73.47 \\ \bottomrule
\end{tabular}
\caption{The accuracy@k scores of all methods in bilingual lexicon induction on \textbf{LEX-C}. The best score for each language pair is bold-faced for the supervised and unsupervised categories, respectively. Languages are paired among \textbf{English~(en), German~(de), Spanish~(es), French~(fr) and Italian~(it)}. "-" means that during the training time, the model failed to converge to reasonable local minimal and hence the result is omitted in the table.}
\label{tab:MUSE-high}
\end{table*}

\subsection{Results in Bilingual Lexicons Induction (Task 1)}
Bilingual lexicon induction is a task to induce a translation in the target language for each query word in the source language. After the query word and the target-language words are represented in the same embedding space (or after our system maps the query word from the source embedding space to the target embedding space), the $k$ nearest target words are retrieved based on their cosine similarity scores with respect to the query vector. 
If the $k$ retrieved target words contain any valid translation according to the gold bilingual lexicon, the translation~(retrieval) is considered successful. The fraction of the correctly translated source words in the test set is defined as $accuracy@k$, which is conventional metric in benchmark evaluations. 


Table~\ref{tab:UnBiLex} shows the accuracy@1 for all the methods on LEX-Z in our evaluation. We can see that our method outperformed the other unsupervised baselines by a large margin on all the eight language pairs. Compared with the supervised methods, our method is still competitive (the best or the second-best scores on four out of eight language pairs), even ours does not require cross-lingual supervision. Also, we notice the performance variance over different language pairs. Our method outperforms all the methods (supervised and unsupervised combined) on the English-Spanish (en-es) pair, perhaps for the reasons that these two languages are most similar to each other, and that the monolingual word embeddings for this pair in the comparable corpus are better aligned than the other language pairs. On the other hand, all the methods including ours have the worst performance on the English-Turkish~(en-tr) pair. Another observation is the performance differences in the two directions of the language pair. For example, the performance of it-en is better than en-it for all methods in table \ref{tab:UnBiLex}. A part of the reason is that there are more unique English words than non-English words in the evaluation set. This would cause direction “xx-en” to be easier than "en-xx" because there are often multiple valid ground truth English translations for each query in “xx”. But the same may not hold for the opposite direction of “en-xx”.  Nevertheless, the relative performance of our method compared to others is quite robust over different language pairs and different directions of translation. 

Table~\ref{tab:MUSE-low} and Table~\ref{tab:MUSE-high} summarize the results of all the methods on the LEX-C dataset.
Several points may be worth noticing. Firstly, the performance scores on LEX-C are not necessarily consistent with those on LEX-Z~(Table \ref{tab:UnBiLex}) even if the methods and the language pairs are the same; this is not surprising as the two datasets differ in query words, word embedding quality, and training-set sizes.  Secondly, the performance gap between the best supervised methods and the best unsupervised methods in both Table~\ref{tab:MUSE-low} and Table~\ref{tab:MUSE-high} are larger than that in Table \ref{tab:UnBiLex}. This is attributed to the large amount of good-quality supervision in LEX-C (5,000 human-annotated word pairs) and the larger candidate size in WE-C ($200,000$ candidates). Thirdly, the average performance in Table~\ref{tab:MUSE-low} is lower than that in Table~\ref{tab:MUSE-high}, indicating that the language pairs in the former are more difficult than that in the latter. Nevertheless, we can see that our method has much stronger performance than other unsupervised methods in Table~\ref{tab:MUSE-low}, i.e., on the harder language pairs, and that it performed comparably with the model by ~\citet{conneau2017word} in Table~\ref{tab:MUSE-high} on the easier language pairs. Combining all these observations, we see that our method is highly robust for various language pairs and under different training conditions.


\begin{table}[ht]
\centering
\resizebox{0.5\textwidth}{!}{%
\begin{tabular}{@{}llllll@{}}
\toprule
 & \begin{tabular}[c]{@{}l@{}}Methods\end{tabular} & de-en & es-en & fa-en & it-en \\ \midrule
\multirow{6}{*}{Supervised} & \citet{mikolov2013efficient} & 0.71 & 0.72 & 0.68 & 0.71 \\
 & \citet{zhang2016ten} & 0.71 & 0.71 & 0.69 & 0.71 \\
 & \citet{xing2015normalized} & 0.72 & 0.71 & 0.69 & 0.72 \\
 & \citet{shigeto2015ridge} & 0.72 & \textbf{0.72} & 0.69 & 0.71 \\
 & \citet{artetxe2016learning} & \textbf{0.73} & \textbf{0.72} & \textbf{0.70} & \textbf{0.73} \\
 & \citet{artetxe2017learning} & 0.70 & 0.70 & 0.67 & 0.71 \\ \midrule
\multirow{3}{*}{Unsupervised} & \citet{conneau2017word} & \textbf{0.71} & \textbf{0.71} & \textbf{0.68} & \textbf{0.71} \\
 & \citet{zhang2017adversarial} & - & - & - & - \\
 & Ours & \textbf{0.71} & \textbf{0.71} & 0.67 & \textbf{0.71} \\ \bottomrule
\end{tabular}%
}
\caption{Performance (measured using Pearson correlation) of all the methods in cross-lingual semantic word similarity prediction on the benchmark data from \citet{conneau2017word}. The best score in the supervised and unsupervised category is bold-faced, respectively. The languages include English~(en), German~(de), Spanish~(es), Persian~(fa) and Italian~(it). "-" means that the model failed to converge to reasonable local minimal during the training process.}
\label{tab:MUSE-corr}
\end{table}

\subsection{Results in Cross-lingual Word Similarity Prediction (Task 2) }
We evaluate models on cross-lingual word similarity prediction~(Task 2) to measure how much the predicted cross-language word similarities match the ground truth annotated by humans. Following the convention in benchmark evaluations for this task, we compute the Pearson correlation between the model-induced similarity scores and the human-annotated similarity scores over testing word pairs for each language pair. A higher correlation score with the ground truth represents the better quality of induced embeddings. All systems use the cosine similarity between the transformed embedding of each query and the word embedding of its paired translation as the predicted similarity score.

%
Table~\ref{tab:MUSE-corr} summarizes the performance of all the methods in cross-lingual word similarity prediction. We can see that the unsupervised methods, including ours, perform equally well as the supervised methods, which is highly encouraging.


\section{Conclusion}
In this paper, we presented a novel method for cross-lingual transformation of monolingual embeddings in an unsupervised manner. By simultaneously optimizing the bi-directional mappings w.r.t. Sinkhorn distances and back-translation losses on both ends, our model enjoys its prediction power as well as robustness, with the impressive performance on multiple evaluation benchmarks.  For future work, we would like to extend this work in the semi-supervised setting where insufficient bilingual dictionaries are available.

\section*{Acknowledgments}
We thank the reviewers for their helpful comments. This work is supported in part by Defense Advanced Research Projects Agency Information Innovation Oce (I2O), the Low Resource Languages for Emergent Incidents (LORELEI) Program, Issued by DARPA/I2O under Contract No. HR0011-15-C-0114, and in part by the National Science Foundation (NSF) under grant IIS-1546329.

\bibliographystyle{apalike}
\bibliography{emnlp2018.bib}

\end{document}